\def\year{2019}\relax
\documentclass[letterpaper]{article} % DO NOT CHANGE THIS
\usepackage{aaai19}  % DO NOT CHANGE THIS
\usepackage{times}  % DO NOT CHANGE THIS
\usepackage{helvet} % DO NOT CHANGE THIS
\usepackage{courier}  % DO NOT CHANGE THIS
\usepackage[hyphens]{url}  % DO NOT CHANGE THIS
\usepackage{graphicx} % DO NOT CHANGE THIS
\usepackage[english]{babel}
\usepackage[utf8]{inputenc}
\usepackage{algorithm}
\usepackage[noend]{algpseudocode}
\usepackage{amsthm}

\newtheorem{theorem}{Theorem}

\usepackage{graphicx}  % DO NOT CHANGE THIS
\title{A-MHA*: Anytime Multi-Heuristic A*}
\author{Ramkumar Natarajan$^\dagger$\thanks{These authors contributed equally.}, Muhammad Suhail Saleem$^\dagger$\footnotemark[1], William Xiao$^\dagger$, \\ \textbf{\Large Sandip Aine$^\ddagger$, Howie Choset$^\dagger$, Maxim Likhachev$^\dagger$} \\ \normalfont $^\dagger$The Robotics Institute, Carnegie Mellon University  \\ \normalfont $^\ddagger$ Apple Inc. }

\begin{document}

\maketitle

\begin{abstract}
% Designing good heuristic functions for graph search requires adequate domain knowledge and tremendous engineering effort. It is often easy to design heuristics that perform well and correlate with the underlying true cost-to-go values in certain parts of the search space but is not admissible (underestimates the optimal cost) throughout the domain. Bounded suboptimal search using several of such partially good but inadmissible heuristics was developed in Multi-Heuristic A* (MHA*). Although MHA* leverages multiple inadmissible heuristics to potentially generate faster suboptimal solution, the original version does not improve the solution over time and also requires careful tuning of inflation factors to obtain maximum utility. In this work, we tackle these two issues by extending the MHA* to an anytime version that adaptively improves the suboptimality bound by parameterizing the inflation factor as a function of current best solution cost. Our work is inspired by the Anytime Nonparametric A* (ANA*) algorithm that eliminates the tuning of inflation factor in suboptimal search. We prove that our careful adaptation of ANA* update rule in the MHA* framework preserves the original bounded suboptimal and completeness guarantees and enhances MHA* to perform in an anytime nonparametric fashion. Furthermore, we demonstrate the results and compare it with bla bla in a bla bla bla domains. AN-MHA* is freely available in Search-Based Motion Planning Library (SMPL).

Designing good heuristic functions for graph search requires adequate domain knowledge. It is often easy to design heuristics that perform well and correlate with the underlying true cost-to-go values in certain parts of the search space but these may not be admissible throughout the domain thereby affecting the optimality guarantees of the search. Bounded suboptimal search using several such partially good but inadmissible heuristics was developed in Multi-Heuristic A* (MHA*) \cite{aine2016multi}. Although MHA* leverages multiple inadmissible heuristics to potentially generate a faster suboptimal solution, the original version does not improve the solution over time. It is a one shot algorithm that requires careful setting of inflation factors to obtain a desired one time solution. In this work, we tackle this issue by extending MHA* to an anytime version that finds a feasible suboptimal solution quickly and continually improves it until time runs out. Our work is inspired from the Anytime Repairing A* (ARA*) algorithm \cite{likhachev2004ara}. We prove that our precise adaptation of ARA* concepts in the MHA* framework preserves the original suboptimal and completeness guarantees and enhances MHA* to perform in an anytime fashion. Furthermore, we report the performance of A-MHA* in 3-D path planning domain and sliding tiles puzzle and compare against MHA* and other anytime algorithms.

\end{abstract}

\section{Introduction}
Real world and real-time planning requires utilizing the limited amount of time available to find a solution that is as close as possible to the optimal one. To that end, anytime algorithms have been developed that can generate a quick suboptimal solution and keep improving it over time. In addition, it is vital to know the quality of such intermediate solutions to decide whether to continue running the planner or to terminate it. Bounded suboptimal search algorithms deals with this problem and provides guarantees on the solution cost. Informed search or heuristic search is an important subclass of these search algorithms that employ underestimates of the true cost-to-go called heuristics to ensure completeness and find bounds on the solution quality. In order to obtain such quantifiable solutions, these heuristics have to satisfy critical properties called admissibility and consistency for all the states in the entire search space. However, crafting heuristics that can obey those properties for large state spaces and high-dimensional planning problems are incredibly hard. 

In many real world scenarios, it is often easy to deduce heuristics that aims to partially solve the bigger problem in hand. Multi-Heuristic A* \cite{aine2016multi} is a recent work that tries to combine such arbitrarily inadmissible heuristics to speed up the search while ensuring strong guarantees. It needs the user to specify the desired suboptimality factor of the output solution prior to beginning the search. This is an extremely tricky step as one requires thorough domain knowledge to strike a balance between runtime and solution quality. In the absence of such domain knowledge, an anytime planner that can rapidly find a low quality solution and steer towards asymptotic convergence is preferred. However, this class of planners can provide guarantees only with consistent heuristics. Anytime Multi-Heuristic A* brings together the best of both worlds. It makes use of multiple inadmissible but informative heuristics supported by an admissible heuristic, to find a suboptimal path as quickly as possible and continues to improve it until the expiration of the allocated time.

% However, the suboptimality bound on the solution found by the search has to be manually fixed before running it. Fixing the suboptimality bound for a search, prior to its commencement is extremely tricky as there is no estimate of how long the search could take to generate an $\epsilon$ suboptimal solution. Therefore, when the suboptimality bound is fixed at a really high value, it is possible that the search finds an extremely suboptimal solution in quick time and makes no attempt to improve it in the remainder of the window that was allocated for the search. However, when this suboptimality bound is fixed at a low value, there is no guarantee that the search will find a solution within the limited allocated time.
% Ideally, the search should find a highly suboptimal solution as quickly as possible and use the rest of the allocated time to improve the solution quality and make it as close to the optimal solution as possible. Anytime repairing A* exactly does that.

 The rest of the paper is organized as follows: In the next section, we briefly go over the related work from anytime and multi-heuristic search. It will be followed by the proposed algorithm and the theoretical properties. We conclude with experimental results and future work.

\section{Related Work}

Efficiency of informed search algorithms, such as A* rely heavily on the accuracy of the heuristic functions. A* with an admissible heuristic is a provable optimal algorithm. However, its runtime and memory requirements often makes it unusable for large state spaces. Weighted A* (WA*) \cite{pohl1970heuristic} can dramatically improve the 
runtime as it inflates the heuristic with a factor $w > 1.0$, providing a greedy flavor to the search. It is 
also a bounded suboptimal algorithm, \emph{i.e.}, the solution obtained is bounded by $w$ times the optimal cost. 
With WA*, the reliance on the heuristic accuracy is magnified (compared to A*), and its 
performance can suffer significantly if the heuristic is subject to large
local minima \cite{wilt2012does}.

Multi Heuristic A* \footnote{We refer to the Shared version of MHA* in the original paper as MHA*} \cite{aine2016multi} alleviates this problem of careful heuristic construction by using multiple heuristics simultaneously to explore a search space. MHA* uses one consistent heuristic and multiple (possibly) inadmissible heuristics, to guide the search around local minima. It often performs better by exploiting the synergy provided by different heuristics, each of which maybe useful in different parts of the search space. MHA* provides guarantees on completeness and bounded suboptimality along with bounds on state expansion (at most 2 expansions per state). There
are variants of the MHA* that improves upon the original by using intelligent scheduling or better bounding \cite{narayanan2015improved}. MHA* and its variants have been recently applied to several complex search problems including fullbody planning \cite{islam2015dynamic}.

A*, WA*, MHA* are all one shot algorithm, as such these do not provide a handle to reason about the trade-off 
between solution quality and runtime. Anytime search algorithms, on the other hand, iteratively improve the solution 
quality, and thus provide the user an opportunity to tradeoff runtime with solution quality. Anytime Repairing A* (ARA*) \cite{likhachev2004ara}, is an anytime search algorithm that uses WA* for a particular iteration, and runs in an anytime mode by decreasing the suboptimality bound over time. ARA* has been successfully applied to many domains, such as 
autonomous cars, mobile manipulation, footstep planning, drones, etc. Other anytime search algorithms include, algorithms based on WA* \cite{richter2010joy} \cite{van2011anytime}, beam search \cite{zhou2005beam}, sliding window search \cite{aine2007awa} etc.

\section{Anytime Multi-Heuristic A* (A-MHA*)}
\textbf{Notations: } Let $s \in \mathcal{S}$ denote the finite set of discrete states over which we search for a path from $s_{start}$ to $s_{goal}$. The search typically proceeds by expanding states to generate successors $s^\prime \in Succ(s)$ based on a priority. The current best cost and the optimal cost to arrive at a state $s$ is denoted by $g(s)$ and $g^*(s)$. $c(s, s^\prime)$ denotes the cost between any two states $s$ and $s^\prime$ connected by an edge. 

As mentioned before, MHA* incorporates a single admissible heuristic $h_0(s)$ and multiple inadmissible heuristics denoted by $h_i(s), \ i=1,...,N$. In this paper, we refer to this admissible search as the anchor search and the other searches as inadmissible searches. We assume that we have access to such admissible and inadmissible heuristics. Let the inflation of the anchor search be $w_1$ and let $w_2$ be the inflation factor to prioritize inadmissible search. Because of the anytime nature of the algorithm, the inflation factors are updated and the found solution is improved over time. They are initialized to $w_1^0$ and $w_2^0$ and updated using $\Delta w_1$ and $\Delta w_2$. With one admissible heuristic and $N$ inadmissible heuristics, the $N+1$ priority queues of expansion are given by $OPEN_0$  and $OPEN_i,  \ i=1,...,N$ respectively. The priority of the states in $OPEN_i$ and $OPEN_0$ are given by $key(s, i) = g(s)+w_1*h_i(s)$. In order to track and prevent re-expansions within a single search improvement routine, we have anchor and inadmissible closed lists and an inconsistent list denoted as $CLOSED_{anch}$, $CLOSED_{inad}$, and $INCONS$ respectively. 

\subsection{Algorithm}
The psuedocode of the proposed algorithm is presented in Algorithm \ref{alg:amha}. The structure of the A-MHA* is similar to anytime search algorithms like ARA* \cite{likhachev2004ara} or ANA* \cite{van2011anytime}. The $\textproc{Main()}$ function consists of the outer loop from which the \textproc{ImprovePath()} function is called with the updated suboptimality bound. The \textproc{ImprovePath()} function is a modified MHA* routine that guarantees $w_1*w_2$ suboptimality and keeps track of inconsistent states to reuse the search results during the next iteration. It consists of two parts, the one that exploits the $w_1$ bounded anchor search (Lines 23-24) and the other that explores the $w_1*w_2$ bounded paths through inadmissible search (Lines 20-21). 

During every iteration of the \textproc{ImprovePath()} function, the option of expanding a state from $OPEN_0$ or $OPEN_i$ is decided depending on their minimum key and $w_2$ (Line 19). We build on the notion of local inconsistency from ARA* \cite{likhachev2004ara} to introduce the inconsistent list in A-MHA* and keep track of the states which were already expanded and whose $g(s)$ is reduced. During the \textproc{Expand()} operation, the state $s$ being expanded is popped from all the $N+1$ queues and checked if it could be a better predecessor (lower $g(s)+c(s,s^\prime)$) to any of the successors $s^\prime$. An update of $g(s^\prime)$ with a better predecessor could cause a local inconsistency between the g-value of $s^\prime$ and all its successors which has to be propagated by putting $s^\prime$ into $OPEN_0$ and $OPEN_i$. In case $s^\prime$ is already expanded (\emph{i.e.} $s^\prime \in CLOSED_{anch}$ or $CLOSED_{inad}$), we delay this propagation by maintaining an $INCONS$ list, an idea developed in ARA* (Lines 8-12). We note that only one $INCONS$ list is needed despite having two $CLOSED$ lists. This can be understood from the observation that all the states added to any $OPEN_i$ are also added to $OPEN_0$ and when a state is expanded from $OPEN_0$, it is never re-expanded from any other $OPEN_i$ in the same \textproc{ImprovePath()} iteration. So if we find a better predecessor to a state which has not been expanded by $OPEN_0$ yet, the priority of the state is updated in $OPEN_0$ and if it has been expanded, it is added to the $INCONS$ list. Thus, after exiting \textproc{ImprovePath()}, the states from both $OPEN_0$ and $INCONS$ are added to both $OPEN_0$ and $OPEN_i$, thereby making sure that all the inconsistent states are tracked by just using $OPEN_0$ and a single $INCONS$ list. 

After exiting from the \textproc{ImprovePath()}, the solution obtained is guaranteed to be $w_1*w_2$ suboptimal (proof in next subsection). Before the next call to \textproc{ImprovePath()}, we move the $INCONS$ states to $OPEN_0$ and $OPEN_i$, re-heap the queues, clear the $CLOSED_{anch}$ \& $CLOSED_{inad}$ and update the $w_1$ and $w_2$ with $\Delta w_1$ and $\Delta w_2$ (Lines 31-40).

\begin{algorithm}[H]
\caption{Anytime Multi Heuristic A* algorithm}
\label{alg:amha}
\begin{algorithmic}[1]
\Procedure{key}{$s,i$}
\State \textbf{return} $g(s) + w_1 * h_i(s)$;
\EndProcedure
\Procedure{Expand}{$s,i$}
\State Remove $s$ from $OPEN_i$ $\forall$ $i = 0,1...N$
\State \textbf{for} each $s^\prime$ in Succ(s)
\State $\>$ \textbf{if} $g(s^{\prime}) > g(s) + c(s,s^{\prime})$
\State $\>$ $\>$ $g(s^{\prime}) = g(s) + c(s,s^{\prime})$
\State $\>$ $\>$ \textbf{if} $s^{\prime}$ in $CLOSED_{anch}$
\State $\>$ $\>$ $\>$ Add $s^{\prime}$ to $INCONS$
\State $\>$ $\>$ \textbf{else}
\State $\>$ $\>$ $\>$ Insert/Update $s^{\prime}$ in $OPEN_0$ with $\textproc{key}(s^{\prime},0)$
\State $\>$ $\>$ $\>$ \textbf{if} $s^{\prime}$ not in $CLOSED_{inad}$
\State $\>$ $\>$ $\>$ $\>$ \textbf{for} $i = 1$ to $n$
\State $\>$ $\>$ $\>$ $\>$ $\>$ \textbf{if} $\textproc{key}(s^{\prime},i) \leq w_2 * \textproc{key}(s^{\prime},0)$ 
\State $\>$ $\>$ $\>$ $\>$ $\>$ $\>$ Insert/Update $s^{\prime}$ in $OPEN_i$ with $\textproc{key}(s^{\prime},i)$
\EndProcedure
\Procedure{ImprovePath}{$ $}
\State \textbf{while} $f(s_{goal}) > w_2 * OPEN_0.Min()$
\State $\>$ \textbf{for} $i =1...N$
\State $\>$ $\>$ \mbox{\textbf{if}$(OPEN_i.Min()\leq w_2*OPEN_0.Min())$}
\State $\>$ $\>$ $\>$ s = $OPEN_i.Top()$
\State $\>$ $\>$ $\>$ \textproc{Expand}$(s,i)$ and Insert s in $CLOSED_{inad}$
\State $\>$ $\>$ \textbf{else}
\State $\>$ $\>$ $\>$ s = $OPEN_0.Top()$
\State $\>$ $\>$ $\>$ \textproc{Expand}$(s,0)$ and Insert s in $CLOSED_{anch}$
\EndProcedure
\Procedure{Main}{$ $}
\State $w_1 = w_1^{0}$; $w_2 = w_2^{0}$; $g(s_{start}) = 0$; $g(s_{goal}) = \infty$; 
% \State $g(s_{start}) = 0$; $g(s_{goal}) = \infty$; 
\State \textbf{for} $i = 0...N$
\State $\>$ $OPEN_i$ = NULL
\State $\>$ Insert $s_{start}$ in $OPEN_i$ with $\textproc{key}(s,i)$
\State \textbf{while} $w_1\ge1$ and $w_2\ge1$
\State $\>$ \mbox{$CLOSED_{anch}$ = $CLOSED_{inad}$ = NULL}
\State $\>$ $INCONS$ = NULL
\State $\>$ \textproc{ImprovePath}()
\State $\>$ Publish current $w_1*w_2$ suboptimal solution
\State $\>$ \textbf{if} $w_1 == 1$ and $w_2 == 1$ 
\State $\>$ $\>$ \textbf{return}
\State $\>$ $w_i = max(w_i - \Delta w_i, 1);$ $i=1,2$
% \State $\>$ $\>$ $w_2 = max(w_2 - \Delta w_2, 1)$
\State $\>$ Move states from $INCONS$ into $OPEN_0$ 
\State $\>$ Copy all states from $OPEN_0$ to $OPEN_i$
\State $\>$ Update the priorities $\forall s \in OPEN_i; \forall i = 0..N$
\EndProcedure
\end{algorithmic}
\end{algorithm}

\subsection{Properties of A-MHA*}
In this subsection, we provide two important properties of A-MHA*. First, we show that the solution provided by any \textproc{ImprovePath()} call is $w_1*w_2$ suboptimal. Next, we show that within each call of \textproc{ImprovePath()}, a state is expanded at-most twice. 

\begin{theorem}
At the exit of \textproc{ImprovePath()}, the cost of the greedy path from $s_{start}$ to any state $s$, $g(s)$, is upper-bounded by $w_1*w_2$ times the cost of the optimal path to the goal $g^*(s_{goal})$. 
\end{theorem}

\begin{proof}
From WA* we know that any state $s$ expanded by the anchor search has a priority (and thereby $g(s)$) lesser than $w_1$ times the optimal path cost $g^*(s_{goal})$ (since the admissible heuristic is an underestimate of the actual cost). By imposing the condition specified in line 19, we restrict the inadmissible expansions only to states whose priority is lesser than $w_1*w_2$ times the optimal cost. Thus, for any state $s \in OPEN_i$ (whose heuristic could be an overestimate) to have a priority lesser than $w_1*w_2$ times the optimal cost, would imply that the current cost to reach that particular state $g(s)$ is lesser than $w_1*w_2$ times optimal path cost. Thus any expansion in the inadmissible search is bounded by $w_1*w_2$ times optimal path cost, including the expansion of the $s_{goal}$. Hence, the path found by \textproc{ImprovePath()} is guaranteed to be $w1*w2$ suboptimal. 

% We borrow the proof from detailed version of MHA* and defer the curious reader to \cite{aine2016multi}. The \textproc{ImprovePath()} can exit after executing lines 21, 24 or at the condition in line 17. If the search in this iteration was performed completely by the anchor portion and exited at line 24, then from the weighted-A* proof it is known that the solution is $w_1$ suboptimal. If the search exited at line 21 or failed to satisfy the condition $f(s_{goal} > w_2*OPEN_0.Min())$, then the solution cost could be $w_1*w_2$ times the optimal solution cost.
\end{proof}

\begin{theorem}
Within a single \textproc{ImprovePath()} call, any state is expanded atmost twice.
\end{theorem}
\begin{proof}
If a state is expanded by a call to \textproc{Expand(s)} from anchor search in line 24, it is added to $CLOSED_{anch}$ and can never be expanded by both inadmissible and anchor search again (note the nested if condition between lines 8 and 15). Next, if a state is expanded by a call to \textproc{Expand(s)} from inadmissible search in line 21, it is added to $CLOSED_{inad}$ and can only be expanded by the anchor search (lines 12-15). Hence, in a single call to \textproc{ImprovePath()}, a state can only be expanded atmost twice. 
\end{proof}

\section{Experimental Results}
We evaluate the performance of A-MHA* on the sliding tiles puzzle and 3D navigation (x,y,orientation) domains and compare it with the performance of other state of the art search algorithms. The experiments are setup similar to those in the original MHA* paper \cite{aine2016multi}, to accurately evaluate the performance of our algorithm and compare in a fair manner.
\begin{figure}[h!]
    \centering
    \includegraphics[width=0.5\textwidth,  trim={4cm 8cm 4cm 8cm}, clip]{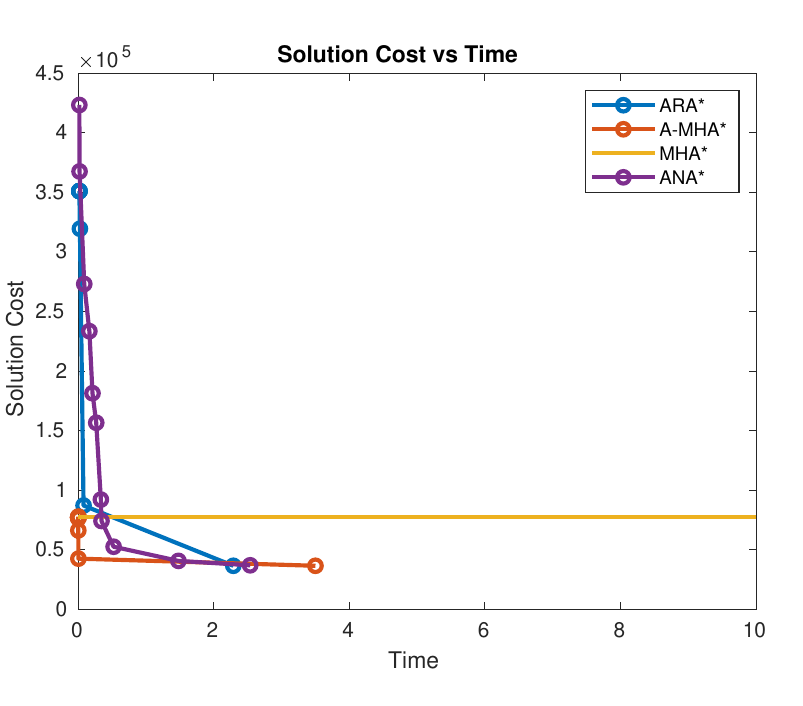}
    \caption{Solution Cost vs Time for 3D planning}
    \label{fig:point_robot}
\end{figure}

\begin{figure}[h!]
    \centering
    \includegraphics[width=0.5\textwidth, trim={4cm 8cm 4cm 8cm}, clip]{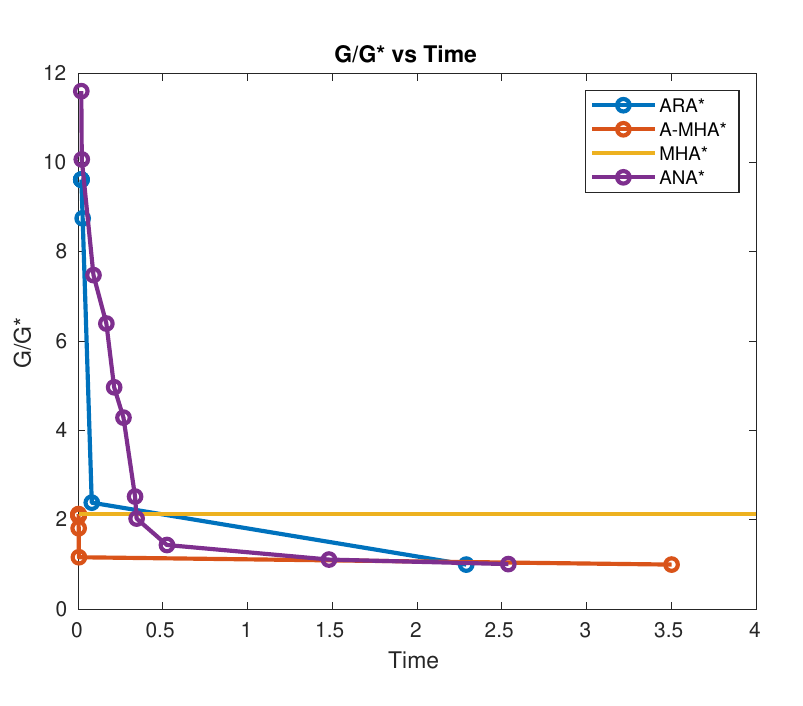}
    \caption{Suboptimality of solution vs Time for 3D planning}
    \label{fig:point_robot_2}
\end{figure}

% height=0.3\textheight,

\subsection{3D Path Planning}
Here, we plan for a polygonal robot with three degrees of freedom (x,y,orientation) in a 2-D planar environment. The plan has to satisfy the minimum turning radius constraints of the robot, which is imposed using motion primitives for generating successors from a state (similar to the lattice type planner \cite{likhachev2009planning}). 

The consistent heuristic, which is the same across all the different planners, is the euclidean distance from the goal. In addition to this, the inadmissible heuristics used for MHA* and A-MHA* include an 8-connected Dijkstra search assuming the robot to have zero size and two other progressive heuristics obtained by running 8-connected Dijkstra search on a map created by blocking the narrow passages (passage width $\leq$ robot size) present in the current map.

From Figure \ref{fig:point_robot} and \ref{fig:point_robot_2}, it is clear that A-MHA* is capable of producing a high quality solution much quicker than the other algorithms, which it continues to improve over time. However, it has to be noted that the usage of inadmissible heuristics delays the convergence to optimal solution. 

\subsection{Sliding Tiles Puzzle}
In this subsection, we present the results for 48 and 63 sliding tiles puzzle, a domain commonly used to evaluate search algorithms. The consistent heuristic for this domain which is widely used in this literature is the sum of the Manhattan Distance ($MD$) and the Linear Conflict ($LC$). Similar to the original MHA* paper, the inconsistent heuristics are a weighted sum of the number of misplaced tiles ($MT$), $MD$ and $LC$, where the weights are randomly generated during execution. 
\begin{table}
\begin{center}
 \begin{tabular}{||c|c|c|c|c||} 
 \hline
 \textbf{Metric} & \textbf{A-MHA*} & \textbf{ARA*} & \textbf{ANA*} & \textbf{MHA*}\\ [0.5ex] 
 \hline\hline
 Success rate & 88.24 & 70.59 & 44.18 & 88.24\\ 
 \hline
 $T_{initial}$ & 9.69 & 18.99 & 23.13 & 9.69\\
 \hline
 $T_{final}$ & 39.70 & 29.82 & 41.73 & 9.69\\
 \hline
 $\epsilon_{initial}$ & 25 & 25 & 3.58e+07 & 25\\
 \hline
 $\epsilon_{final}$ & 4 & 5.1667 & 3.31 & 25\\ [1ex] 
 \hline
\end{tabular}
\caption{Average statistics for 50 instances of 63 tile sliding puzzle: $T_{initial}$ - Time to produce the first solution; $T_{final}$ - Time to produce the final solution; $\epsilon_{initial}$ - Reported Initial suboptimality bound;  $\epsilon_{final}$ - Reported final suboptimality bound.}\label{tab:results}

\end{center}
\end{table}

\begin{table}
\begin{center}
 \begin{tabular}{||c|c|c|c|c||} 
 \hline
 \textbf{Metric} & \textbf{A-MHA*} & \textbf{ARA*} & \textbf{ANA*} & \textbf{MHA*}\\ [0.5ex] 
 \hline\hline
 Success rate & 100 & 75 & 75 & 100\\ 
 \hline
 $T_{initial}$ & 17.51 & 15.05 & 42.41 & 17.51\\
 \hline
 $T_{final}$ & 42.90 & 31.075 & 111.42 & 17.51\\
 \hline
 $\epsilon_{initial}$ & 25 & 25 & 2.95e+07 & 25\\
 \hline
 $\epsilon_{final}$ & 7.33 & 8.77 & 7.90 & 25\\ [1ex] 
 \hline
\end{tabular}
\caption{Average statistics for 50 instances of 48 tile sliding puzzle: $T_{initial}$ - $T_{initial}$ - Time to produce the first solution; $T_{final}$ - Time to produce the final solution; $\epsilon_{initial}$ - Reported Initial suboptimality bound;  $\epsilon_{final}$ - Reported final suboptimality bound.}\label{tab:results_2}

\end{center}
\end{table}

From Table. \ref{tab:results} and \ref{tab:results_2}, we understand that A-MHA* has the highest success rate (number of instances for which the puzzle was solved within a limited time) and clearly outperforms the other algorithms consistently. For 63 tile sliding puzzles, similar to the previous domain, A-MHA* produces high quality solutions  in much lesser time, while converging to the final suboptimality bound slower. However, for the 48 tiles sliding puzzle environment, the performance of A-MHA* is comparable to that of ARA*. This is because the effect of the additional heuristics in a smaller environment might not have the same impact as it did in a bigger/complex environment and hence might not outweigh the overhead in maintaining several heaps.

\section{Conclusion}
We presented an anytime version of MHA*, while preserving completeness, expansion and suboptimality guarantees. Experimental results in two different domains strongly favor A-MHA* over the other anytime algorithms. This simple algorithm brings together the benefits of both ARA* and MHA*, thereby maximizing the performance in terms of both solution quality and run-time. One interesting future work could be nonparametrize the inflation factors of A-MHA* to circumvent the tedious process of tuning them.

\bibliographystyle{aaai} 
\bibliography{references}
\end{document}